%% file: paper.tex
\documentclass[letterpaper, 10 pt, conference]{ieeeconf}  

\IEEEoverridecommandlockouts                              

\usepackage{cite} 

\usepackage{multicol}

\usepackage{url}

\makeatletter
\let\NAT@parse\undefined
\makeatother

\usepackage{amsmath} 
\usepackage{amsfonts} 
\usepackage{mathrsfs}

\usepackage{amsthm} 
\usepackage{thmtools}
\usepackage{amssymb} 
\usepackage{graphicx} 
\usepackage{enumerate} 
\usepackage{todonotes} 
\usepackage{xspace} 
\usepackage{tabularx}
\usepackage{extarrows}
\usepackage{soul}
\usepackage{subcaption}
\usepackage{siunitx}
\usepackage{wrapfig}
\usepackage{tikz}
\tikzset{>=latex}
\usetikzlibrary{arrows,automata,positioning}

\usepackage{bm}
\usepackage{mdframed}

\usepackage[bookmarks=true,hidelinks]{hyperref}
\input{defs.tex}

\newcommand\shortenXor[2]{\ifdefined\arxiv #1\else #2\fi}

\newcommand\shorten[1]{\shortenXor{#1}{}}

\newcommand\shortvspace[1]{\ifdefined\arxiv\else \vspace*{#1}\fi}
\newcounter{tecounter}
\DeclareFontFamily{U}{mathb}{\hyphenchar\font45}
\DeclareFontShape{U}{mathb}{m}{n}{<5> <6> <7> <8> <9> <10> gen * mathb
<10.95> mathb10 <12> <14.4> <17.28> <20.74> <24.88> mathb12}{}
\DeclareSymbolFont{mathb}{U}{mathb}{m}{n}
\DeclareMathSymbol{\rcirclearrow}{0}{mathb}{'367}

\title{\LARGE \bf On nondeterminism in combinatorial filters
\shortvspace{-8pt}}

\author{Yulin Zhang and Dylan A. Shell%
\thanks{Yulin Zhang is with the Dept. of Computer Science,
The University of Texas, Austin, TX, USA. 
{\tt\small yulin@cs.utexas.edu}.
Dylan A. Shell is with Dept. of Computer Science \& Engineering, 
Texas A\&M University, College Station, TX, USA
{\tt\small dshell@tamu.edu}.
}
\thanks{This work was supported by the National Science Foundation through awards
\href{http://nsf.gov/awardsearch/showAward?AWD_ID=1453652}{IIS-1453652},
\href{https://nsf.gov/awardsearch/showAward?AWD_ID=1849249}{IIS-1849249},
and
\href{https://nsf.gov/awardsearch/showAward?AWD_ID=2034097}{IIS-2034097}.
}}

\begin{document}

\maketitle
\begin{abstract}
The problem of combinatorial filter reduction arises from \shorten{questions of }resource
optimization in robots; it is one specific way in which automation can help to
achieve minimalism, to build better\shorten{, simpler} robots.  This paper contributes a
new definition of filter minimization that is broader than its antecedents, 
allowing filters (input, output, or both) to be nondeterministic.  
This changes the problem considerably. 
Nondeterministic filters \shortenXor{are able to}{may} re-use states to obtain\shorten{, essentially,}
more `behavior' per vertex. We show that the gap in size can be significant
(larger than polynomial), 
suggesting such cases will generally be more challenging than 
deterministic problems. Indeed, this is supported by the core
\shorten{computational }complexity result established in this paper: 
producing nondeterministic minimizers is \pspacehard.
\shortenXor{
The hardness separation for minimization which exists between
deterministic filter and 
deterministic automata, thus, does not 
hold for the nondeterministic case.}{The hardness separation for minimization existing between
deterministic filter and 
automata, thus, fails
to hold for the nondeterministic case.}

%



%
%
%
%
%
%

\end{abstract}

\section{Introduction}

\shortenXor{As robots become ever more complex, it is natural to turn to computational tools to help automate their design and
fabrication processes.  }{With increasingly complex robots,
one naturally turns to computational tools to help automate design processes.}  This leads directly to the practical question of how to
reduce a robot's resource footprint.  Minimizing resources causes one to reason
about their necessity, which furnishes more fundamental insights about the
underlying information requirements of particular robot tasks~\cite{donald95}.
This paper focuses on \shorten{the problem of }minimizing state in \emph{combinatorial
filters}~\cite{lavalle10sensing}, discrete variants of the probabilistic
estimators so widely used in robotics~\cite{thrun02probablistic}.  While their
minimization problem is easy to formulate (to wit: reduce the number of
states while preserving input--output behavior), it is computationally hard
to solve.

Combinatorial filter reduction was first introduced as an open question by
Tovar et al.~\cite[pg.~12]{tovar2014combinatorial}.  They introduced the
scenario in \fig~\ref{fig:diag_donut} to exemplify the problem: two agents
wander in a circular world, and three sensor beams (producing symbols `\Sa',
`\Sb', and `\Sc', \shortenXor{respectively}{resp.}) partition the environment into sector-shaped
regions (labeled $0$, $1$, $2$).  The beams detect if an agent crosses the
dividing line but senses neither the agent's identity nor direction of
motion.  With the agents starting in some known configuration, the task is,
given a sequence of sensor readings (i.e., a string of \Sa's, \Sb's, \Sc's), to
determine whether the pair are in the same sector or not.  This
problem may be solved via a filter, a finite transition system akin to a Moore
machine transducer whose vertices bear an output (or color).  Starting at the
initial state, one traces the input string forward to produce a sequence of
colors that represent estimates.  When every string gives a solitary tracing,
the filter is \emph{deterministic}.  One wonders: what is the smallest filter
for tracking the co-location of our two agents?  Only 4 states are required.
\shortenXor{The answer appears in}{See} \fig~\ref{fig:reduced_donut}; the minimal filter is
deterministic.  Other than human nous (how Tovar \& friends did it), one may
produce a minimal instance by starting with the filter obtained by directly
transcribing of the problem, and applying a reduction algorithm. 

\begin{figure}[t!]
\shortvspace{-8pt}

	\begin{subfigure}[b]{0.35\columnwidth}
	\begin{center}
		\includegraphics[scale=0.8]{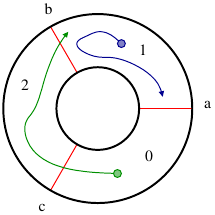}
        \shortvspace{-6pt}
		\caption{\label{fig:diag_donut}}
	\end{center}
	\end{subfigure}
	\begin{subfigure}[b]{0.55\columnwidth}
	\begin{center}
		\includegraphics[scale=0.37]{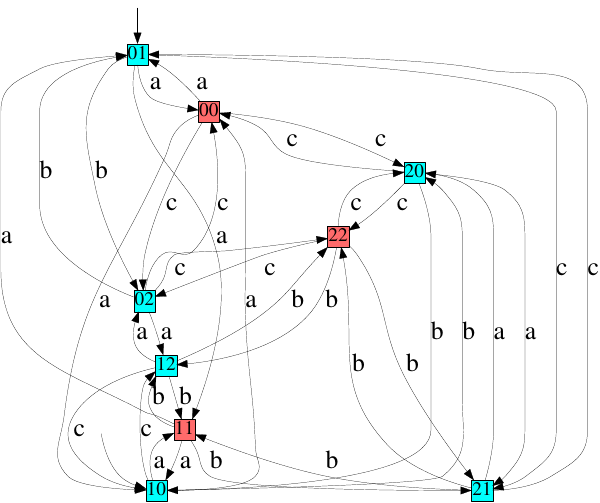}
        \shortvspace{-6pt}
		\caption{\label{fig:nd_donut}}
	\end{center}
	\end{subfigure}
	\begin{subfigure}[b]{0.35\columnwidth}
	\begin{center}
		\includegraphics[scale=0.23]{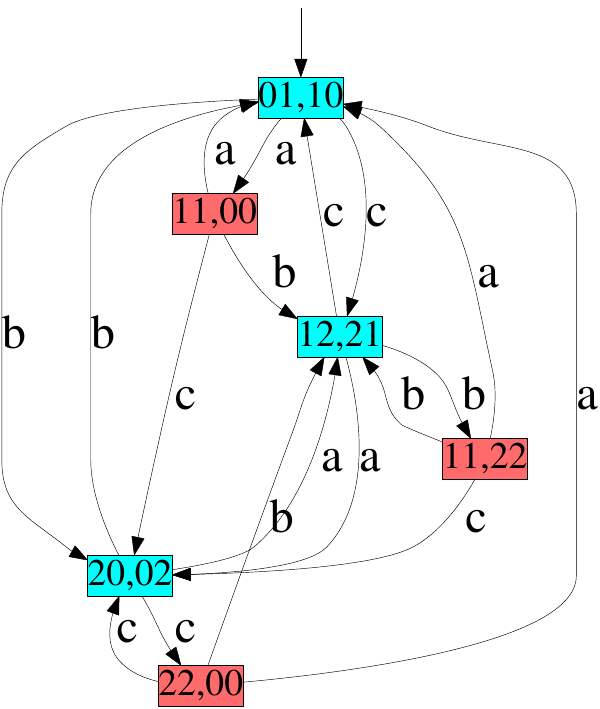}
        \shortvspace{-6pt}
		\caption{\label{fig:sdt_donut}}
	\end{center}
	\end{subfigure}
	\hfill
	\begin{subfigure}[b]{0.55\columnwidth}
	\begin{center}
		\includegraphics[scale=0.23]{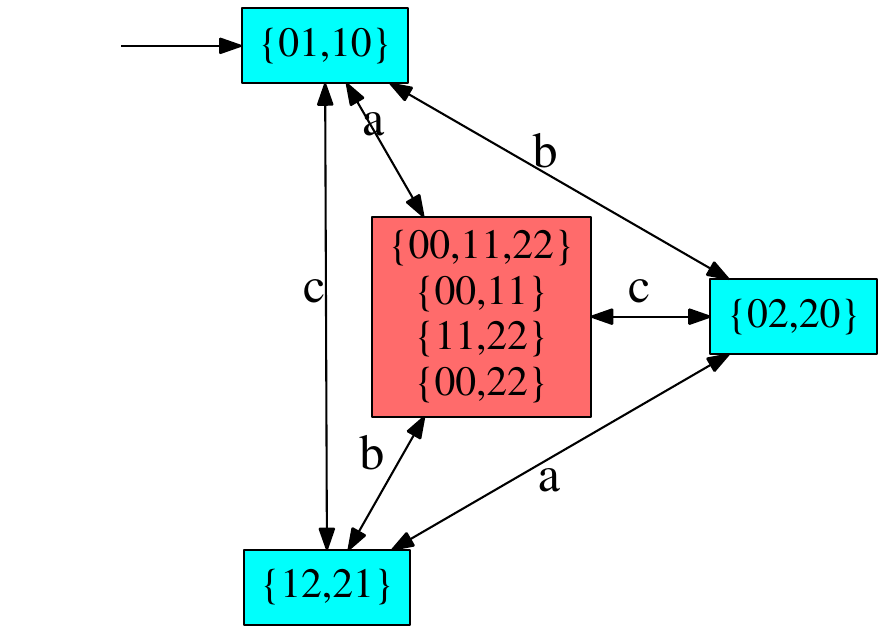}
        \shortvspace{-6pt}
		\caption{\label{fig:reduced_donut}}
	\end{center}
	\end{subfigure}
	\caption{
    \shortenXor{}{\small}
Combinatorial filter minimization, as originally motivated
by~\cite{tovar2014combinatorial}: (a) A pair of agents move in a circular 
environment
with three beam sensors, partitioned into regions,
indexed by numbers $0$, $1$ and $2$. Letters `\Sa', `\Sb', and `\Sc' denote
observations from each of the three beams. (b) A nondeterministic filter to
estimate whether these two agents are in the same region (red) or not (cyan).
(c)~The deterministic version of the same. (d) A \emph{minimal} deterministic filter.
\label{fig:donut}}
\shortvspace{-24pt}
\end{figure}

Most prior work on combinatorial filters, including all research on filter
reduction until now~\cite{o2017concise,saberifar2017combinatorial,rahmani2020integer,zhang2020cover,zhang2020AcceleratingCF}, concerns deterministic filters.  The present paper, in its
first part, presents a compelling practical case for the utility of filter
minimization methods that accommodate nondeterminism.  The second part
of this paper
examines the hardness of the minimization problem for filters with
nondeterministic inputs, including finding both deterministic and
nondeterministic minimizers for nondeterministic input filters. 
We show that, under commonly held computational complexity assumptions,
these problems are harder than the deterministic case. In what follows, we leverage hardness
results from automata theory to establish these facts, which has the important
added benefit of leading to a broader and clearer understanding of the
relationship between filter and automata minimization.

\section{The value of nondeterminism in minimizing combinatorial filters}

Existing research on combinatorial filter
reduction~~\cite{o2017concise,saberifar2017combinatorial,rahmani2020integer,zhang2020cover,zhang2020AcceleratingCF}
only deals with deterministic input filters and deterministic output filters
(or minimizers).  To understand the implications of this,  \shorten{let's return to the
example in \fig~\ref{fig:donut} in some detail.}{let's return to \fig~\ref{fig:donut} in some detail.}  To arrive at the 4-state
minimizer, we begin with the diagram in \fig~\ref{fig:diag_donut}.  Using the
assumption of continuous motion and beginning at a state representing the initial
agent configuration, we trace all possible events forward, coloring the conditions
encountered appropriately (red for together; cyan otherwise).  The result,
\fig~\ref{fig:nd_donut}, is not deterministic. To apply a minimization
algorithm, the filter must be converted to an equivalent one that is.  The process
of determinizing produces a filter (\fig~\ref{fig:sdt_donut}) that can then
be fed into a minimization method to yield
\fig~\ref{fig:reduced_donut}.  This procedure goes from 9 states, to 6,
before reaching 4\shortenXor{, finally.}{.}

\shortenXor{By way of contrast,}{But now} consider the nondeterministic 5-state filter in
\fig~\ref{fig:nd_wiki}. To find a minimal filter, it can be determinized (via
a power set construction\shorten{, see}~\cite{setlabelrss}) to track the $2^4=16$
distinct information states shown in \fig~\ref{fig:std_wiki}. 
\shortenXor{Once that is
minimized,}{Once minimized,} it gives the deterministic filter in \fig~\ref{fig:min_wiki}.  The
growth in the number of vertices, caused by the need to determinize for the
minimization algorithm, indicates trouble.  Not only does the set increase
exponentially, but this much larger object becomes the input for an exponential
cost algorithm (\shortenXor{since}{as} the problem is \nphard~\cite{o2017concise}).  Double
trouble.

To by-pass this expansion, one requires filter reduction methods that are
able to consume nondeterministic filters directly as input.  Looking again at
\figs.~\ref{fig:nd_wiki} and~\ref{fig:min_wiki}, the dramatic compression that
cancels the extreme expansion raises some questions. Do large deterministic
instances arising from small nondeterministic ones really induce hard
minimization problems? Or are they instead structured in some special \shortenXor{(sparse,
say, or otherwise low-density)}{(sparse or low-density)} form, \shortenXor{reflecting conservation of }{conserving }underlying
information? Computational complexity
\shortenXor{can provide clues: for example,}{provides clues: e.g.,}
in characterizing the space requirements of direct nondeterministic filter to
deterministic minimizer computation.

\begin{figure}[b]
	\vspace{-0.4cm}
	\begin{subfigure}[b]{0.2\linewidth}
	\begin{center}
		\includegraphics[scale=0.35]{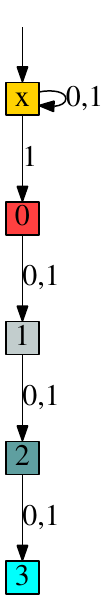}
		\caption{\label{fig:nd_wiki}}
	\end{center}
	\end{subfigure}
	\begin{subfigure}[b]{0.4\linewidth}
	\begin{center}
		\includegraphics[scale=0.35]{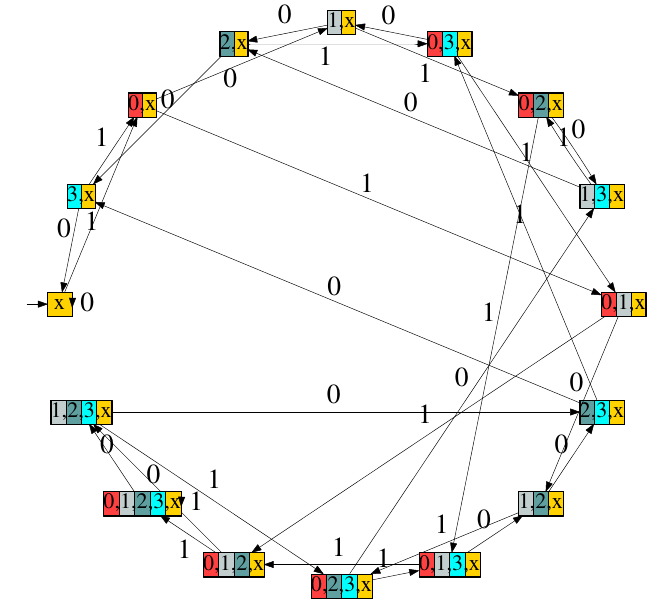}
		\caption{\label{fig:std_wiki}}
	\end{center}
	\end{subfigure}
	\hfill
	\begin{subfigure}[b]{0.28\linewidth}
	\begin{center}
		\includegraphics[scale=0.28]{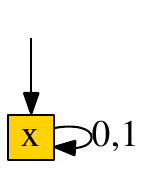}
		\vspace{1.3cm}
		\caption{\label{fig:min_wiki}}
	\end{center}
	\end{subfigure}
	\caption{
    \shortenXor{}{\small}
    (a) A nondeterministic filter. (b) A deterministic form
obtained via the power set construction. (c) A minimal filter.
\label{fig:wiki_example}} 
\end{figure}

If nondeterminism can be of added value as input to a minimization algorithm,
what about as its output?  In finite automata minimization, the smallest
nondeterministic automata can be smaller than any deterministic one.
Typical examples exploit the fact that accepting a string in the nondeterministic
automaton requires that \emph{some} tracing arrive at an accepting state. 
For filters, analogous instances fail owing to their differing semantics
(stated formally in the next section).
The analogous fact, however, does hold. A small 
example suffices to show this: the deterministic input filter given in
\fig~\ref{fig:det_input} has \num{19} states, but can be reduced to a
deterministic minimizer with size \num{14}. This minimizer has a single color
selected for each of the \num{5} leaf states which have a choice, and then merges 
identically colored leaves.
However, the filter can be shrunk still further: 
\fig~\ref{fig:nd_output} gives a 
nondeterministic minimizer possessing 
only \num{13} states.  Nondeterminism, then,
provides extra freedom
that can be exploited to further
reduce filter size. 

\shortenXor{To summarize: ($i$)~nondeterminism in the input allows minimization to proceed
directly on models of certain problems, potentially saving on expensive
intermediate steps; ($ii$)~permitting nondeterminism in the filters produced
as output can deliver greater compression.
Thus, nondeterminism may be of considerable practical importance.}{
To summarize: nondeterminism may be practical importance for 
two reasons: ($i$)~nondeterminism in the input allows minimization to proceed
directly on models of certain problems, potentially saving on expensive
intermediate steps; ($ii$)~permitting nondeterminism in the filters produced
as output can deliver greater compression.}


\begin{figure}
    \shortvspace{5pt}

	\begin{subfigure}[b]{\linewidth}
	\begin{center}

	\scalebox{0.5}{
	\input{./figure/merged_deterministic_input_small.tex} }
	\caption{\label{fig:det_input}}
	\end{center}
	\end{subfigure}
    \shortvspace{-5pt}

	\begin{subfigure}[b]{\linewidth}
	\begin{center}
	\scalebox{0.5}{
	\input{./figure/smo_minimizer_small.tex}}
	\end{center}
    \shortvspace{-5pt}
	\caption{\label{fig:nd_output}}
	\end{subfigure}
	\caption{
    \shortenXor{}{\small}
    (a) A \num{19}-state deterministic filter that has no deterministic minimizer with
	fewer than \num{14} states.  (b) A minimal nondeterministic \num{13}-state
	minimizer for the filter above.\label{fig:mobi}}
	\vspace{-0.4cm}
\end{figure}
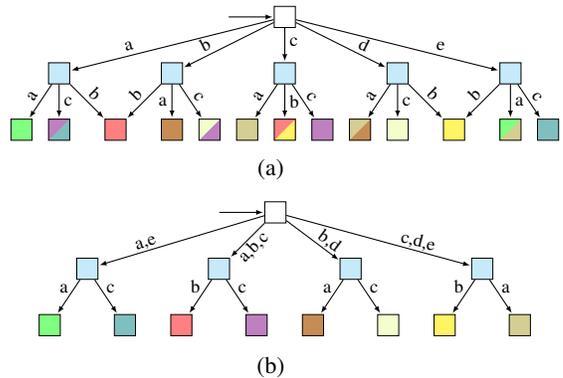

\section{Combinatorial filters and their minimization}
We first give our model of combinatorial filters:
\begin{definition}[procrustean filter~\cite{setlabelrss}]
A \defemp{procrustean filter}, \defemp{p-filter} or \defemp{filter} for short,
is a tuple $(V, V_0, Y, \tau, C, c)$ where $V$ is the set of states, $V_0$ is
the set of initial states, $Y$ is the set of of observations, $\tau: V\times
V\rightarrow 2^Y$ is the transition function, $C$ is the set of outputs
(or colors), and $c: V\to 2^{C}\setminus\{\emptyset\}$ is the output function.
\end{definition}
We write the states, initial states and observations for a filter
$\structure{F}$ as $V(\structure{F})$, $V_0(\structure{F})$ and
$Y(\structure{F})$. A filter $\structure{F}=(V, V_0, Y, \tau, C, c)$ is
\defemp{deterministic}, if $|V_0|=1$ and for every $v_1, v_2, v_3\in V$ with
$v_2\neq v_3$, $\tau(v_1, v_2)\cap \tau(v_1, v_3)=\emptyset$. Otherwise, we say
$\structure{F}$ is \defemp{nondeterministic}. 
A filter can also be viewed as a graph with states being its vertices, and
transitions being directed edges.

In a filter $\structure{F}=(V, V_0, Y, \tau, C, c)$, an observation sequence (or
a string) $s=y_1 y_2\dots y_n\in Y^*$ \defemp{reaches} a state $w$ from state
$v$, if there exists a sequence of states $w_0, w_1, \dots, w_n$ in
$\structure{F}$, such that $w_0=v, w_n=w$, and $\forall i\in \lbrace
1,2,\dots,n\rbrace$, $y_i\in \tau(w_{i-1}, w_i)$. In a filter $\structure{F}$, for
every state $v$, if there exists a string that reaches $v$ from some initial
state, then we say $\structure{F}$ is \defemp{trim}. Any filter that is not
trim can be made so by removing the states that are not reached by any
string from the initial states. \shortenXor{Without any loss of generality, we
consider filters that are trim.}{
We consider filters that are trim, w.l.o.g.
}

We collect the set of all states reached by $s$ from some initial state $v_0\in
V_0$, and denote it as $\reachedv{\structure{F}}{s}$. Specifically, for the empty
string $\epsilon$, we have $\reachedv{\structure{F}}{\epsilon}=V_0$. If no 
states are reached by some string from any initial state, then we say that 
string \defemp{crashes} on $\structure{F}$. The set of strings that do not crash
on $\structure{F}$ is called the \defemp{interaction language} of
$\structure{F}$, and is written as $\ILanguage{\structure{F}}=\lbrace s\in Y^{*}\mid
\reachedv{\structure{F}}{s}\neq\emptyset\rbrace$.  The \defemp{output of
string} $s$ on filter $\structure{F}$ is the set of outputs (or colors) of all
states reached by $s$ from some initial state, and is written as
$\reachedc{\structure{F}}{s}=\cup_{v\in \reachedv{\structure{F}}{s}} c(v)$. 

\shortenXor{When minimizing some filter, we are interested in reduced filters
that simulate the given filter in terms of outputs on its strings:}{In minimizing a filter, we are interested in reduced filters
that simulate the given filter in terms of outputs on its strings:}%
\begin{definition}[output simulating]
Let $\structure{F}$ and $\structure{F}'$ be two filters, then $\structure{F}'$
\defemp{output simulates} $\structure{F}$ if the following
properties hold: ($i$) language inclusion: $\ILanguage{\structure{F}}\subseteq
\ILanguage{\structure{F}'}$; ($ii$) output consistency: $\forall s\in
\ILanguage{\structure{F}}$, $\reachedc{\structure{F}'}{s}\subseteq
\reachedc{\structure{F}}{s}$.
\end{definition}

\shortenXor{Intuitively, this}{This} requires that $\structure{F}'$ be capable of processing all the inputs
which $\structure{F}$ can, and produce outputs 
that $\structure{F}$ could. The input set is no smaller; the set of outputs
no larger.


\shortenXor{We want to search for a minimal filter that output simulates the original filter:}{We wish to find minimal filters:}
\ourproblem{\textbf{Filter Minimization (\pfm)}}
{A filter $\structure{F}$.}
{A filter $\structure{F}^{\dagger}$ with fewest states, such that $\structure{F}^{\dagger}$ output
simulates $\structure{F}$.
}
This is a generalization of its deterministic version in the
work~\cite{zhang2020cover,zhang2020AcceleratingCF}, which dealt only with
deterministic input and deterministic minimizer. We use `{\sc pf}' to denote the
fact that both the input and output of this problem can be general
nondeterministic p-filters, and use `{\sc m}' for minimization. Additionally,
we designate the problem of producing a \emph{deterministic minimizer} for a
nondeterministic input filter `\pfdm', a four-letter word where `{\sc
d}' stands for deterministic.

\section{Background and preliminaries}
We make use of some known facts from automata theory.


A finite automaton (NFA) is a tuple $(Q, Q_0, \Sigma, \delta, A)$, where $Q$,
$Q_0$, $\Sigma$, $\delta$, $A$ are the states, initial states, alphabet
(observations), transition function, and accepting states. 
Both filters and NFAs are similar, both being transition structures. 
Different from a filter, an NFA
$\structure{A}$ has accepting states not outputs (colors). 
For automata, we are interested in the strings that reach some accepting
states, which we term the accepting language $\ALanguage{\structure{A}}$. An
NFA with a singleton $Q_0$ and deterministic transition structure is also
called a DFA.

Here are some results from automata theory:
\begin{lemma}[\shortenXor{NFA equivalence}{\!\!}\cite{stockmeyer1973word}] Given two 
NFAs $\structure{A}$ and $\structure{B}$, it is \pspacecomplete to check if
$\ALanguage{\structure{A}}=\ALanguage{\structure{B}}$.  
\label{lem:nfa_equiv}
\end{lemma}
\shortvspace{-8pt}

\begin{lemma}[\shortenXor{NFA universality}{\!\!}\cite{meyer1972equivalence}] 
For a given NFA $\structure{A}=(Q, Q_0, \Sigma, \delta, A)$, it is \pspacecomplete
to check whether \mbox{$\ALanguage{\structure{A}}=\Sigma^{*}$}.  
\label{lem:nfa_univ}
\end{lemma}
\shortvspace{-8pt}

\begin{lemma}[\shortenXor{DFA union universality}{\!\!}\cite{rampersad2012computational}]
Given a set of DFAs $\{\structure{A}_1$, $\structure{A}_2$, \dots,
$\structure{A}_n\}$ with common alphabet $\Sigma$, it is \pspacecomplete to check
if $\cup_{1\leq i\leq n} \ALanguage{\structure{A}_i}=\Sigma^{*}$.  
\label{lem:Udfa_univ}
\end{lemma}

\section{Complexity of filter minimization}
\shortenXor{In this section, we leverage prior hardness results from automata
theory to show that the problems of finding minimizers, including deterministic
and nondeterministic minimizers, for nondeterministic input filters are hard.
Specifically, we will show that the decision version of \pfm is \pspacecomplete,
and the decision version of \pfdm is \pspacehard.
}
{We are now ready to show that finding minimizers for nondeterministic input filters is hard.}

The decision problem of the \pfm problem is:
\decproblem{\textbf{P-filter Minimization (\pfmdec)}}
{A filter $\structure{F}$ and $k\in \PositiveNaturals$.}
{\textsc{Yes} \shortenXor{if}{only if} there exists \shortenXor{a p-filter}{some} $\structure{F}^{\dagger}$ with no more than $k$
states, such that $\structure{F}^{\dagger}$ \oss $\structure{F}$. \shorten{\textsc{No} otherwise.}
}
\shortenXor{Analogously, we denote the decision version of \pfdm as \pfdmdec.}
{Analogously, \pfdmdec is the decision version of \pfdm.}

\subsection{The hardness of \pfm and \pfdm}
Now, we will show that the decision versions of \pfm and \pfdm
are, respectively, \pspacecomplete and \pspacehard. Consequently, both \pfm and \pfdm are
\pspacehard.

\shortenXor{It is helpful, as a first step, to introduce a product operator for
constructing of the product graph of two filters in polynomial time. This
operator will be used to check the output simulation requirement.}{
To check the output simulation requirement in polynomial time, the following
filter product will be helpful.}
\begin{definition}[tensor product]
Given \shorten{two }filters $\structure{F}_1=(V^1, V^1_0,  Y^1, \tau^1, C^1, c^1)$ and
$\structure{F}_2=(V^2, V^2_0, Y^2, \tau^2, C^2, c^2)$, their product,
a graph denoted 
$(\structure{F}_1\fprod \structure{F}_2$), 
is constructed to capture
strings in $\ILanguage{\structure{F}_1}$ \shortenXor{as follows:}{via:}
\begin{enumerate}
\item List all pairs of vertices in $V(F_1)\times (V(F_2)\cup\{\placeholder\})$,
where \placeholder is a placeholder for an empty vertex. 
\item Mark vertex $(v,w)$ an initial state in graph 
$(\structure{F}_1\fprod \structure{F}_2)$. 
\item Build a transition from $(v, w)$ to $(v', w')$ under label $y$ if
$y\in\tau^1(v, v')$ and $y\in \tau^2(w, w')$. Notice that if $y$ is not an
outgoing label of vertex $v$, then we say $y\in \tau^1(v,\placeholder)$.  
\shortenXor{
\item Remove the pairs that are not reached from any initial state in
$(\structure{F}_1\fprod \structure{F}_2)$.}
{
\item Remove the pairs reached from any initial state.}
\end{enumerate}
\end{definition}
\shortenXor{Notice that the}{The} tensor product of two filters is a transition structure with initial
states, i.e., a graph.

\begin{lemma}
\label{lm:fm_np}
\pfmdec is in \pspace.
\end{lemma}
\shortvspace{-10pt}
\begin{proof}
\shortenXor{
We show first that representing and searching for a filter takes polynomial
space, and then that only polynomial space is needed to 
ascertain whether a filter output simulates
$\structure{F}=(V, V_0, Y, \tau, C, c)$.} {Two steps: 
polynomial space suffices 
(1) to represent and search for a filter, and (2) to 
ascertain whether a filter output simulates
$\structure{F}$.}
\shortenXor{Since}{For (1), since}  \pfmdec requires we encode filters of 
size $k$, we need to keep track of at most $k^2$ transitions, at most
$|Y|$ labels for each transition, at most $|C|$ colors for each state, and at
most $k$ initial states.  The space needed to enumerate 
output filters is $O(k^2\times |Y|+k\times |C|)$.   

\shortenXor{
To show that it also takes polynomial space to check whether a filter
$\structure{F'}$ output simulates the filter $\structure{F}$ that was provided
as input, we need to check the language inclusion property and then output
consistency.

First, we will show that it takes polynomial space to establish that
$\ILanguage{\structure{F}}\subseteq \ILanguage{\structure{F'}}$ by converting
it to the problem of NFA equivalence, which is in \pspace.
We form 
product graph $\structure{G}=\structure{F}\fprod \structure{F'}$. If there is
no vertex $(v, \placeholder)$ in $\structure{G}$ such that $v\in
V(\structure{F})$, then we claim that $\ILanguage{\structure{F}}\subseteq \ILanguage{\structure{F'}}$
since every string reaching a vertex in $\structure{F}$ also reaches some vertex
in $\structure{F}'$. If there exists some such a vertex $(v,\placeholder)$,
then we must determine whether the strings reaching $(v,\placeholder)$ also reach
some vertex in $\structure{F}'$. We build an NFA $\structure{A}$ from
$\structure{G}$ by treating all states $\lbrace(v,\placeholder)\mid v\in
V(\structure{F})\rbrace$
as accepting states. Next, we create another NFA, $\structure{B}$, from
$\structure{F}'$ by treating every state in $\structure{F}'$ as accepting.
Then we want to show that strings reaching every
$(v,\placeholder)$ are accepted by $\structure{B}$.  Or, in other words,
whether NFAs $\structure{A}$ and $\structure{A}\cap \structure{B}$
are equivalent (where $\cap$ is automata intersection). 
Creating automata $\structure{A}$ and $\structure{A}\cap
\structure{B}$, and showing their equivalence is 
in \pspace.
}{For (2), we must verify both language inclusion  and output
consistency.
To show $\ILanguage{\structure{F}}\subseteq \ILanguage{\structure{F'}}$,
form 
product graph $\structure{G}=\structure{F}\fprod \structure{F'}$. 
If there is
no vertex $(v, \placeholder)$ in $\structure{G}$ such that $v\in
V(\structure{F})$, then $\ILanguage{\structure{F}}\subseteq \ILanguage{\structure{F'}}$
since every string reaching a vertex in $\structure{F}$ also reaches some vertex
in $\structure{F}'$. 
If there exists some such a vertex $(v,\placeholder)$,
then we must determine whether the strings reaching $(v,\placeholder)$ also reach
some vertex in $\structure{F}'$. We build an NFA $\structure{A}$ from
$\structure{G}$ by treating all states $\lbrace(v,\placeholder)\mid v\in
V(\structure{F})\rbrace$
as accepting states. Next, construct a second NFA, $\structure{B}$, from
$\structure{F}'$ by treating every state in $\structure{F}'$ as accepting.
Then we must show that strings reaching every
$(v,\placeholder)$ are accepted by $\structure{B}$, i.e.,
whether $\structure{A}$ and $\structure{A}\cap \structure{B}$
are equivalent (where $\cap$ is automata intersection). 
Creating automata $\structure{A}$ and $\structure{A}\cap
\structure{B}$, and, via Lemma~\ref{lem:nfa_equiv}, showing their equivalence is 
in \pspace.
}

\shortenXor{
Secondly, verifying output consistency also needs only polynomial space.  Begin
by removing the states of the form $(v,\placeholder)$ from $\structure{G}$.
Then, for every state $(v,w)$ in $\structure{G}$ such that $c(v)\not\supseteq c(w)$, 
to output simulate, for every output $o\in c(w)\setminus c(v)$, strings reaching
$(v, w)$ must reach some state $u$ in $\structure{F}$ with $o\in c(u)$. Otherwise, $o$
is not a legal output for some string, and $\structure{F'}$ doesn't output
simulate $\structure{F}$. To see whether $o$ is a legal output, we build an NFA 
$\structure{M}$ from
$\structure{G}$ by treating $(v,w)$ as accepting states, and another NFA
$\structure{N}$ from $\structure{F}$ by treating the states with color $o$ as
accepting states. If $\ALanguage{\structure{M}}\subseteq
\ALanguage{\structure{N}}$, then $o$ is safe. If every $o\in c(w)\setminus c(v)$
is safe, then the output of $\structure{F}'$ is consistent on that of
$\structure{F}$. Otherwise, $(v,w)$ is an evidence of violation for output
consistency. This procedure also takes polynomial amount of space. 
}
{
Verifying output consistency also needs only polynomial space.  
Remove the states of the form $(v,\placeholder)$ from $\structure{G}$,
then, for every state $(v,w)$ in $\structure{G}$ such that $c(v)\not\supseteq c(w)$, 
to output simulate, for every output $o\in c(w)\setminus c(v)$, strings reaching
$(v, w)$ must reach some state $u$ in $\structure{F}$ with $o\in c(u)$.
Otherwise $o$ is not a legal output for some string. To see
whether $o$ is a legal output, we build an NFA $\structure{M}$ from
$\structure{G}$ by treating $(v,w)$ as accepting states, and another NFA 
$\structure{N}$ from $\structure{F}$ by treating the states with color $o$ as
accepting states. If $\ALanguage{\structure{M}}\subseteq
\ALanguage{\structure{N}}$, then $o$ is safe. If every $o\in c(w)\setminus
c(v)$ is safe, then the output of $\structure{F}'$ is consistent on that of
$\structure{F}$. Otherwise, $(v,w)$ is certificate for violation of output
consistency. This procedure takes polynomial space.  }
\shortenXor{ Therefore, \pfmdec is in \pspace.}{} \end{proof}

\begin{lemma}
\label{lm:fm_nph}
\pfmdec is \pspacehard.
\end{lemma}
\shortvspace{-10pt}
\begin{proof}
We give a polynomial time reduction from NFA universality (Lemma~\ref{lem:nfa_univ}) to \pfmdec. To
show the accepting language of a given NFA $\structure{A}=(Q, Q_0, \Sigma,
\delta, A)$ is $\Sigma^{*}$, we first create a filter $\structure{F}$
from $\structure{A}$ \shortenXor{as follows:}{in polynomial time as follows:}
\begin{enumerate}
\item Add the states, transitions, initial states of $\structure{A}$ to the
states, transitions, initial states of $\structure{F}$.
\item Add a new initial state $v$ to $\structure{F}$, with a self loop bearing
all labels $\Sigma$ from $\structure{A}$.
\item Add a new vertex $w$ to $\structure{F}$. 
For every state in $\structure{F}$ arising from an 
accepting state in $\structure{A}$, 
add a transition to $w$ under some
new label $z$, where $z \not\in \Sigma$. 
\item Add one more vertex $u$ to $\structure{F}$, and a transition from $v$ to
$u$ under $z$.
\item Color $u$ blue, the all other vertices green.
\end{enumerate}
\shorten{This procedure takes polynomial time.}

Now, the interaction language for this filter is $\Sigma^{*}z$. Further, the
outputs of strings $\ALanguage{\structure{A}}z$ are both green and blue, while
the outputs for
the strings $(\Sigma^{*}\setminus \ALanguage{\structure{A}})z$ are blue only. 

If $\ALanguage{\structure{A}}$ is $\Sigma^{*}$, then the minimal filter for
$\structure{F}$ has only one green state and it has a self loop bearing
$\Sigma\cup\lbrace z\rbrace$. If
$\ALanguage{\structure{A}}$ is not $\Sigma^{*}$, then there there exists some
string $s\not\in \ALanguage{\structure{A}}$ where $s$ only outputs green,
and $sz$
only outputs blue.  There must, therefore, be at least two states (one colored
green, and one colored blue) in its minimizer. As a consequence, if the
minimizer of $\structure{F}$ has only one state, then
$\ALanguage{\structure{A}}$ is $\Sigma^{*}$. Otherwise,
$\ALanguage{\structure{A}}$ is not $\Sigma^{*}$.

Therefore, we get a polynomial time reduction from NFA universality
to \pfmdec. \pfmdec is \pspacehard since NFA universality is
\pspacecomplete.
\end{proof}

\begin{lemma}
\label{lm:fm_dec}
\pfmdec is \pspacecomplete.
\end{lemma}
\shortvspace{-10pt}
\begin{proof}
Combine Lemmas~\ref{lm:fm_np} and~\ref{lm:fm_nph}.
\end{proof}

\begin{theorem}
\pfm is \pspacehard.
\end{theorem}
\shortvspace{-10pt}
\begin{proof}
This is a direct consequence of Lemma~\ref{lm:fm_dec}.
\end{proof}

Having considered the case where both the input and the minimizer may be
nondeterministic, next we show that limiting nondeterminism to only the input
filter (what we dubbed \pfdmdec earlier) still retains its hardness.

\begin{theorem}
\label{lm:fdm_dec_hard}
\pfdmdec is \pspacehard. 
\end{theorem}
\shortvspace{-10pt}
\begin{proof}
We show the \pfdmdec is \pspacehard by reducing the DFA union universality
problem (Lemma~\ref{lem:Udfa_univ}) to \pfdmdec. Given a set of DFAs, $\structure{A}_1,
\structure{A}_2,\dots, \structure{A}_n$, let the union of their alphabet be
$\Sigma$. The DFA union universality problem is to check
$\ALanguage{\structure{A}_1}\cup \ALanguage{\structure{A}_2}\cup \dots\cup
\ALanguage{\structure{A}_n}=\Sigma^{*}$. 
For each DFA $\structure{A}_i$, we 
first, we construct a DFA $\structure{A}'_i$, such that
$\ALanguage{\structure{A}'_i}=\ALanguage{\structure{A}_i}$ and
$\ILanguage{\structure{A}'_i}=\Sigma^{*}$:
\begin{enumerate}
\item Initialize $\structure{A}'_i$ as a copy of $\structure{A}_i$.
\item If $\ILanguage{\structure{A}'_i}\neq \Sigma^{*}$, then add a trap state
$v'$ with a self loop bearing all labels in $\Sigma$ for each DFA
$\structure{A}'_i$. For each state $w'$ in $\structure{A}'_i$ and every
outgoing event $y\in \Sigma$, if $y$ crashes when traced from $w'$, build a
transition from $w'$ in 
 $\structure{A}'_i$ to the trap state $v'$ under $y$. 
\item Make all states corresponding to accepting states in $\structure{A}_i$ 
 the accepting states for $\structure{A}'_i$.
\end{enumerate} 
Next, we build an NFA $\structure{B}'$ as the union of all these
$\structure{A}'_i$'s, 
so as to have $\ALanguage{\structure{A}_1}\cup
\ALanguage{\structure{A}_2}\cup \dots\cup
\ALanguage{\structure{A}_n}=\ALanguage{\structure{B}'}$. 
Additionally, no strings in $\Sigma^{*}$ crash on $\structure{B}'$. The task,
then, is to check  whether $\ALanguage{\structure{B}'}=\Sigma^{*}$ holds or not.  

To do so, create a filter $\structure{F}$ from $\structure{B'}$ as follows:
\begin{enumerate}
\item Add the initial states, states, transitions of $\structure{B}'$ to \shorten{those
of }$\structure{F}$.
\item Color the copies of the accepting states in $\structure{B'}$ green, and
the copies of the non-accepting states red.
\item Add one more state, and color it green. Make this state the 
destination reached from 
one goal state under a fresh symbol $z$ (i.e., where
$z$ is not a symbol from $\Sigma$). 
\end{enumerate}
By adding the new symbol $z$, we known that there is some string ending with
$z$ which outputs only green in $\structure{F}$.

Supposing $\structure{H}$ is a deterministic minimizer of $\structure{F}$,
there are two cases. First, if $\structure{H}$ is a one-state filter, then it
must be green because there are some strings that must output only green.
\shortenXor{But then, since the
one-state filter output simulates $\structure{F}$,}{Then} every string in $\Sigma^{*}$
must output at least green in $\structure{F}$.  Hence, every string in
$\Sigma^{*}$ must reach the accepting states in $\structure{B}'$, and
we conclude $\ALanguage{\structure{B}'}=\Sigma^{*}$. 
Alternatively, if $\structure{H}$ has more
than one state, then there is at least one green state and one red state.
(Otherwise, $\structure{H}$ is not minimal.) \shortenXor{As a consequence,}{Hence,}
there must be some string in $\Sigma^{*}$ that can only output red. Those
strings with only red output never reach the accepting states in
$\structure{B}'$. So, consequently, $\ALanguage{\structure{B}'}\neq \Sigma^{*}$. 

The procedure to solve \pfdmdec involves checking whether there is a
one-state minimizer for $\structure{F}$. If there is such a minimizer, then
the accepting language of
the union of all DFAs is $\Sigma^{*}$.  Otherwise, it is not. 
Having given a polynomial time procedure to reduce the DFA union universality
problem to \pfdmdec, which is itself known to be \pspacecomplete, shows that
\pfdmdec is \pspacehard.  
\end{proof}

\shortenXor{
Since \pfdm can be no easier than its decision version, we
can claim that \pfdm is \pspacehard in terms of space complexity.
}
{
Since \pfdm can be no easier than its decision version, 
\pfdm is \pspacehard in terms of space complexity.
\shortvspace{-3pt}
}
\begin{theorem}
\pfdm is \pspacehard.
\end{theorem}

\subsection{Is \pfdmdec\xspace \pspacecomplete?}

It seems natural to suppose that \pfdmdec is simpler than \pfmdec and should
also be in \pspace since the problem is narrower, focusing on more constrained
(deterministic) minimizers.  However, from the perspective of space
consumption, this needn't be the case. 

To help elucidate, it's useful to introduce some lightweight notation
for the problems showing their inputs and outputs explicitly. 
We denote a minimization problem  that converts something of type~A to a
corresponding minimal instance of type~B as A\tom B.
We will compare and contrast filters with automata, and write 
deterministic and nondeterministic instances as $\textbf{Det}$ and $\textbf{NDet}$ respectively.
\begin{figure}
	\begin{center}
		\includegraphics[width=1\columnwidth]{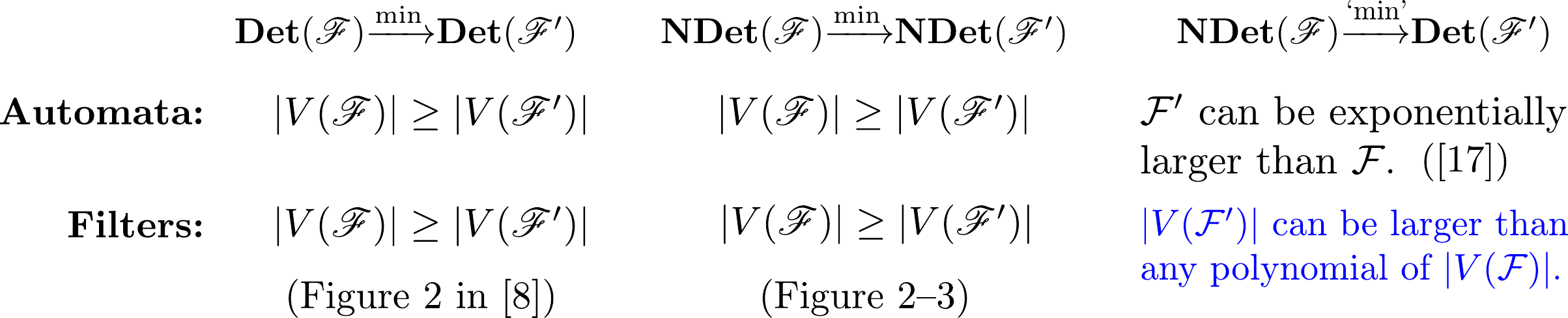}
	\end{center}
\vspace*{-6pt}
	\caption{
    \shortenXor{}{\small}
    An overview about the size of the minimizer for both automata
minimization and filter minimization problems.%
\label{fig:size}}
\shortvspace{-16pt}
\end{figure}

\fig~\ref{fig:size} gives an overview about how the size may change during the
process of filter minimization and automata minimization.  
From results known in the literature (and the previous figures in this paper), we know that
for $\textbf{Det}(\mathscr{F})\tom \textbf{Det}(\mathscr{F'})$ and
$\textbf{NDet}(\mathscr{F})\tom \textbf{NDet}(\mathscr{F'})$ in both filter
minimization and automata minimization, the resulting minimal object is
always smaller than (or equal to) the size of the object provided as input.
But this need not be true when turning a $\textbf{NDet}(\mathscr{F})$ element
into a $\textbf{Det}(\mathscr{F'})$ one, as the minimizer is constrained to be deterministic
and it can be larger than the nondeterminsitic input. In particular, for
automata minimization, the `minimizer' (DFA) can be exponentially larger than
the input automata (NFA)~\cite{meyer1971economy}. For filter minimization, take
the problem shown in \fig~\ref{fig:mobi}, and exchange the roles of the two
graphs: the \num{14}-state deterministic minimizer of \fig~\ref{fig:det_input}
is also a deterministic minimizer for the \num{13}-state nondeterministic filter
shown in \fig~\ref{fig:nd_output}. There, the deterministic minimizer has one
more state than its nondeterministic input filter. 

How big can the difference actually be? We give a construction for a family
of filters demonstrating that the size of the deterministic minimizer may grow
so that its size is beyond any polynomial in the input size (highlighted as blue
in \fig~\ref{fig:size}). 
First, we make a nondeterministic input filter, then we follow that by giving
its deterministic minimizer.

\begin{figure}

\begin{subfigure}[b]{\linewidth}
\begin{tikzpicture}[shorten >=1pt,node distance=1cm, on grid,auto, font=\tiny, initial text = ] 
\tikzset{every state/.style={minimum size=0pt, inner sep=1pt, rectangle}}
   \node[state, initial, fill=black, minimum size =0.3 cm] (q0)   {}; 
   \node[state] (q11) [above right=3.2cm and 2cm of q0] {$q^1_1$}; 
   \node[state] (q12) [right=of q11] {$q^2_1$}; 
   \node[state, fill=red] (p11) [below=0.8cm of q11] {$p^1_1$}; 
   \node[state, fill=orange] (p12) [right=of p11] {$p^2_1$}; 

   \node[state] (q21) [below=0.8cm of p11] {$q^1_2$}; 
   \node[state] (q22) [right=of q21] {$q^2_2$}; 
   \node[state] (q23) [right=of q22] {$q^3_2$}; 
   \node[state, fill=red] (p21) [below=0.8cm of q21] {$p^1_2$}; 
   \node[state, fill=orange] (p22) [right=of p21] {$p^2_2$}; 
   \node[state, fill=yellow] (p23) [right=of p22] {$p^3_2$}; 

   \node[state] (q31) [below=0.8cm of p21] {$q^1_3$}; 
   \node[state] (q32) [right=of q31] {$q^2_3$}; 
   \node[state] (q33) [right=of q32] {$q^3_3$}; 
   \node[state] (q34) [right=of q33] {$q^4_3$}; 
   \node[state] (q35) [right=of q34] {$q^{5}_3$}; 
   \node[state, fill=red] (p31) [below=0.8cm of q31] {$p^1_3$}; 
   \node[state, fill=orange] (p32) [right=of p31] {$p^2_3$}; 
   \node[state, fill=yellow] (p33) [right=of p32] {$p^3_3$}; 
   \node[state, fill=green] (p34) [right=of p33] {$p^4_3$}; 
   \node[state, fill=blue] (p35) [right=of p34] {$p^{5}_3$}; 

   \node[draw=none] (p41) [below=0.6cm of p31] {\qquad \large \dots }; 

   \node[state] (q51) [below=1.2 cm of p31] {$q^1_r$}; 
   \node[state] (q52) [right=of q51] {$q^2_r$}; 
   \node[state] (q53) [right=of q52] {$q^3_r$}; 
   \node[state] (q54) [right=of q53] {$q^4_r$}; 
   \node[draw=none] (q55) [right=of q54] {\large \dots}; 
   \node[state] (q56) [right=of q55] {$q^m_r$}; 
   \node[state, fill=red] (p51) [below=0.8cm of q51] {$p^1_r$}; 
   \node[state, fill=orange] (p52) [right=of p51] {$p^2_r$}; 
   \node[state, fill=yellow] (p53) [right=of p52] {$p^3_r$}; 
   \node[state, fill=green] (p54) [right=of p53] {$p^4_r$}; 
   \node[draw=none] (p55) [right=of p54] {\large \dots}; 
   \node[state, fill=violet] (p56) [right=of p55] {$p^m_r$}; 

    \path[->] 
    (q0) edge node {a} (q11)
    (q11) edge node [swap] {a} (q12)
    (q12) edge[bend right] node[above] {a} (q11)

    (q11) edge node [right] {$\Sx_1$} (p11)
    (q12) edge node [right] {$\Sx_1$} (p12)

    (q0) edge node {a} (q21)
    (q21) edge node {a} (q22)
    (q22) edge node {a} (q23)
    (q23) edge[bend right] node[above] {a} (q21)

    (q21) edge node [right] {$\Sx_2$} (p21)
    (q22) edge node [right] {$\Sx_2$} (p22)
    (q23) edge node [right] {$\Sx_2$} (p23)

    (q0) edge node {a}  (q31)
    (q31) edge node {a} (q32)
    (q32) edge node {a} (q33)
    (q33) edge node {a} (q34)
    (q34) edge node {a} (q35)
    (q35) edge[bend right=25] node[below] {a} (q31)

    (q31) edge node [right] {$\Sx_3$} (p31)
    (q32) edge node [right] {$\Sx_3$} (p32)
    (q33) edge node [right] {$\Sx_3$} (p33)
    (q34) edge node [right] {$\Sx_3$} (p34)
    (q35) edge node [right] {$\Sx_3$} (p35)

    (q0) edge node[below] {a}  (q51)
    (q51) edge node {a} (q52)
    (q52) edge node {a} (q53)
    (q53) edge node {a} (q54)
    (q54) edge node {a} (q55)
    (q55) edge node {a} (q56)
    (q56) edge[bend right=25] node[below] {a} (q51)

    (q51) edge node [right] {$\Sx_r$} (p51)
    (q52) edge node [right] {$\Sx_r$} (p52)
    (q53) edge node [right] {$\Sx_r$} (p53)
    (q54) edge node [right] {$\Sx_r$} (p54)
    (q56) edge node [right] {$\Sx_r$} (p56)

    (q0) edge node[above] {a}  (p41);
    \end{tikzpicture}
    \caption{
    \shortenXor{}{\small}
    A nondeterministic filter with $n$ rows, where the number of white
states at the $i^{\rm th}$ row is the $i^{\rm th}$ prime number. 
\label{fig:nd_prime}}
\end{subfigure}

\begin{subfigure}[b]{\linewidth}
\begin{tikzpicture}[shorten >=1pt,node distance=1cm, on grid,auto, font=\tiny, initial text = ] 
\tikzset{every state/.style={minimum size=0.3 cm, inner sep=1pt, rectangle}}
   \node[state, initial above, fill=black, minimum size =0.3 cm] (r0)   {}; 
   \node[state] (r1) [right=1cm of r0] {$r_1$}; 
   \node[state] (r2) [right=2cm of r1] {$r_2$}; 
   \node[state] (r3) [right=2cm of r2] {$r_3$}; 
   \node[draw=none] (r4) [right=1.2cm of r3] {\large \dots}; 
   \node[state] (r5) [right=1.2cm of r4] {$r_{q}$}; 

   \node[state, fill=red] (q1) [below=of r1] {$p^1_r$}; 
   \node[state, fill=orange] (q2) [below=of r2] {$p^2_r$}; 
   \node[state, fill=yellow] (q3) [below=of r3] {$p^3_r$}; 
   \node[draw=none] (q4) [below=of r4] {\large \dots};
   \node[state, fill=violet] (q4) [below=of r5] {$p^4_r$};

    \path[->] 
    (r0) edge node {a} (r1)
    (r1) edge node {a} (r2)
    (r1) edge node[left] {$\Sx_1,\Sx_2,\dots, \Sx_r$} (q1)
    (r2) edge node {a} (r3)
    (r2) edge node[left, pos=0.24] {$\Sx_1,\Sx_2,\dots, \Sx_r$} (q2)
    (r3) edge node {a} (r4)
    (r3) edge node[left, pos=0.8] {$\Sx_1$} (q1)
    (r3) edge node[anchor=center, pos=0.2] {$\Sx_2,\dots, \Sx_r$} (q3)
    (r4) edge node {a} (r5)
    (r5) edge node[left, pos=0.8] {$\Sx_1$} (q2)
    (r5) edge node[left, pos=0.7] {$\Sx_2$} (q3)
    (r5) edge node[left] {$\Sx_r$} (q4)
    (r5) edge[bend right=20] node[below] {a} (r1);
    \end{tikzpicture}
    \caption{A minimal deterministic filter for (a), where the total number of
white states is the product of first $n$ prime numbers.\label{fig:min_det}}
\end{subfigure}
    \caption{
    \shortenXor{}{\small}
    An example to show that the number of states in the deterministic
minimizer is larger than polynomial size of the nondeterministic input
filter.\label{fig:prime}}
\shortvspace{-18pt}
\end{figure}
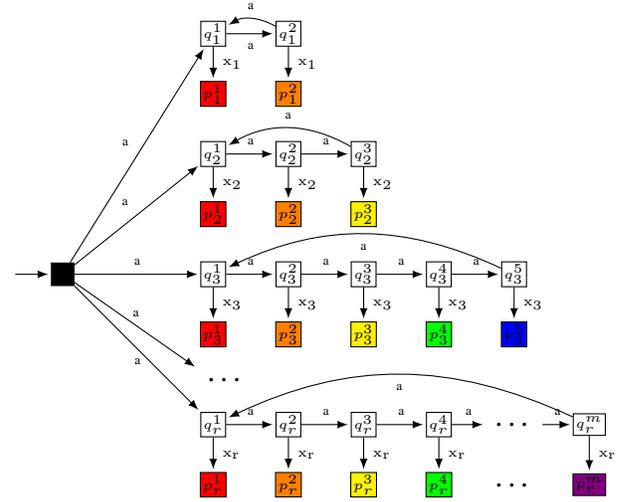
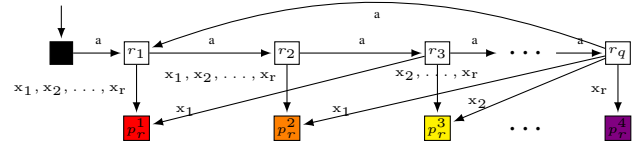

\begin{construction} 
\label{const:prime}
Fix some natural number $r$, and construct the nondeterministic input filter with
$r$ rows depicted in \fig~\ref{fig:nd_prime}. 
Create a cycle of white states under `\Sa' where the number of white states at
row $i\in \PositiveNaturals$ is the cycle of length $p_i$, the $i^{\rm th}$
prime number.  
For example, the number of white states in rows $1$, $2$, $3$ are $2$, $3$, $5$,
respectively.  Create a black initial state that connects, via `\Sa', to one state at
each of these $r$ rows. At each row, starting from the state connected with the
initial one, add a transition to a new child state. 
Color the child with a color from the color list $[o_1,
o_2, \dots, o_{p_r}]$,
that excludes both black and white. Each child state is colored as the first
one that is not chosen in the row. 

An equivalent deterministic filter, shown in \fig~\ref{fig:min_det}, is
produced via the power set construction. Notice that all states in the
nondeterministic filter of \fig~\ref{fig:nd_prime} reached by
a common string share the same color, so there are no choices for each
state in the deterministic filter: 
every state must be colored to correspond.
(Part of the next lemma will show this to be
the minimizer.)

\end{construction}

\begin{lemma}
\label{lem:size_is_big}
The deterministic minimizer of a nondeterministic input filter can exceed any
polynomial of the input size.
\end{lemma}
\shortvspace{-10pt}
\begin{proof}
First we argue that the deterministic form of the nondeterministic input filter
from Construction~\ref{const:prime} is a deterministic minimizer, and then show that
the gap between the size of the nondeterministic input filter and its
deterministic minimizer is larger than any polynomial of the input size.

The deterministic filter shown in \fig~\ref{fig:min_det} is already a minimal
one for the filter depicted in \fig~\ref{fig:nd_prime}. 
The $n$ colors must be included as they are each produced by some string; the
white vertices could only be merged if there was a common divisor in the cycle
lengths, but the cycle lengths are all distinct primes.  Hence, no pair of
states in \fig~\ref{fig:min_det} can be merged since they either have
different outputs or disagree on the outputs of their common extensions. 

Let $n$ be the total number of states in this nondeterministic input filter.
Then we have $n=2\cdot S(r)+1$, where $S(r)=\sum_{i=1}^r p_i$ is summation of
the first $r$ prime numbers, and Bach and Shallit\,\cite{bach1996algorithmic} have shown
that $S(r)\sim\frac{1}{2}r^2\ln r$ holds asymptotically. 
When $1<r$, $n = r^2\ln r+1 < r^3$.  

Let $z$ represent the total number of states in the deterministic filter. Then
we have $z=1+P(r)+p_r$, where $P(r)=\prod_{i=1}^{r} p_i$ is the primorial,
i.e., the product of the first $r$ prime numbers. According to the prime
number theorem and the first Chebyshev function, we have that $P(r)\sim
e^{(1+o(1))r\log r}$ and $p_r\sim r\log r$ holds asymptotically~\cite{jameson2003prime}. Hence,
\shortenXor{
\begin{equation*}
\begin{aligned}
z&=1+P(r)+p_r\\
&=1+e^{(1+o(1))r\log r}+r\log r >e^r \quad\textrm{for large }r.
\end{aligned}
\end{equation*}}{

\shortvspace{-10pt}
{
\begin{equation*}
\begin{aligned}
\qquad z&=1+P(r)+p_r\\
&=1+e^{(1+o(1))r\log r}+r\log r >e^r \quad\textrm{for large }r.
\end{aligned}
\end{equation*}}
\shortvspace{-5pt}
}

Since $r>\sqrt[3]{n}$, we have $z>e^{\sqrt[3]{n}}$. So we write this lower bound of
$z$ as $f(n)=e^{\sqrt[3]{n}}=\sum_{m=0}^{\infty} \frac{n^{\frac{m}{3}}}{m!}$
(Taylor series). 
\smallskip

Now consider any polynomial of $n$ of degree $k$ and write it as $g(n,
k)=\sum_{m=0}^k \alpha_m n^m$.  Let $c=\max\{\alpha_0,
\alpha_1,\dots, \alpha_k\}$. If $n>c\cdot(k+1)$, then we have for all $i\leq
k$, the coefficients have $\alpha_i n^i<c\cdot n^k$, and the sum
$\sum_{i=0}^{i=k} \alpha_i n^i<n^{k+1}$. 

To bring the two bounds in relation to one another: when $n>(3k+6)!$, then
$f(n)>\frac{n^{\frac{3k+6}{3}}}{(3k+6)!}>\frac{n^{\frac{3k+6}{3}}}{n}=n^{k+1}$,
Hence, $f(n)>n^{k+1}$ if $n>(3k+6)!$.  Thus for  $n>\max\{c\cdot (k+1),\,
(3k+6)!\}$, we have that $z>f(n)>n^{k+1}>g(n,k)$.  This is true for any $k$, so
the size of the deterministic minimizer, $z$, is larger than any polynomial of
$n$.
\shorten{Therefore, the number of states in the deterministic minimizer can exceed any
polynomial of the input size.}\end{proof}

One implication of the preceding example is that:

\begin{lemma}
\label{lm:beyondp}
\pfdm is not in \p.
\end{lemma}
\shortvspace{-10pt}
\begin{proof}
Since the size of the minimizer can be larger than any polynomial of the input
size, it takes more than polynomial time to output the minimizer. Therefore,
\pfdm $\not\in$ \p.
\end{proof}

Then, considering time complexity further, we can conclude that \pfdm is strictly \nphard.
\begin{theorem}
\pfdm is \nphard, but not in \p. 
\end{theorem}
\begin{proof}
The deterministic input to deterministic output filter minimization
problem, the decision problem form of which is \npcomplete~\cite{o2017concise},
is properly contained in \pfdm (one just happens to select an input that is 
deterministic). 
We have that \pfdm is \nphard, and 
combining with
Lemma~\ref{lm:beyondp}, we can conclude that \pfdm is
strictly \nphard.
\end{proof}

To summarize, Construction~\ref{const:prime} and Lemma~\ref{lem:size_is_big} show that the gap between the
size of the deterministic minimizer can be larger than polynomial of the input
size.  It indicates that constructing and storing the deterministic minimizer
in its entirety to determine its size would disqualify \pfdmdec from \pspace.
Of course, other cleverer means may exist, so whether \pfdmdec is \pspace (as a
consequence, \pfdmdec is \pspacecomplete) or not remains an open question.

\section{A comparison between automata minimization and filter minimization}

With the preceding hardness results for filter minimization problems
established, we now compare them with the hardness of automata minimization in
\fig~\ref{fig:overview}.
It is worthwhile to try distill intuition for a couple of reasons:
firstly, the automata hardness results were used in the arguments above, so
their connection might seem obvious at first blush. But the
notion of equivalence beween two automata is quite different from
that beween two filters,  as, importantly, are specific requirements on
interaction
vs. accepting languages.  Secondly, recall that the initial supposition that 
deterministic filter and deterministic automata minimization problems were
identical, was wrong.


In the first column of \fig~\ref{fig:overview}
($\textbf{Det}(\mathscr{F})\tom \textbf{Det}(\mathscr{F'})$): automata
minimization problem ($\mathscr{F}$ is a DFA) can be solved efficiently by
identifying Myhill--Nerode equivalence classes~\cite{mikolajczak1991algebraic},
while the decision version of filter minimization problem ($\mathscr{F}$ is
a filter) is \npcomplete.  The main reason for this hardness separation between
these two problems is the extra degree of freedom (DOF) for filter minimization.
Filters can choose to assign \emph{any} output for the strings that crash
\shorten{in the filter} (informally, we call this DOF~\MakeUppercase{\romannumeral 1}). To exploit this
degree of freedom
optimally, it is equivalent to searching for a minimum clique cover in the
compatibility graph of the input filter~\cite{zhang2020cover}, which makes the
problem computationally hard. 

For the other two columns of \fig~\ref{fig:overview}: As we consider
nondeterminism in the input or output object, the hardness separation
between automata minimization and filter minimization disappears. Informally
speaking, it appears that the hardness arising from DOF
\MakeUppercase{\romannumeral 1} is dominated by other sources of complexity. 
When nondeterminism appears in both input and output, i.e.,
$\textbf{NDet}(\mathscr{F})\tom
\textbf{NDet}(\mathscr{F'})$, the decision problems of both filter minimization
and automata minimization are \pspacecomplete\!\!\cite{jiang1993minimal}.  For
both, there could be multiple states simultaneously reached by the same string
(DOF \MakeUppercase{\romannumeral 2}) though it takes no more than polynomial
space to check the outputs of those states.  Though  both are \pspacecomplete,
the problems differ in the degrees of freedom they have---though, clearly, this
difference is not enough to manifest as a hardness gap.
On the one hand, nondeterministic filter minimization ($\mathscr{F}$ is a
non-deterministic filter) has DOF~\MakeUppercase{\romannumeral 1} while
nondeterminsitic automata 
minimization ($\mathscr{F}$ is an NFA) does not.  On the other,
non-deterministic filter minimization requires all outputs of all states reached
by the string be constrained, whereas non-deterministic automata minimization
can choose to accept the states or not, as long as at least one
is accepted~(DOF~\MakeUppercase{\romannumeral 3}), 

If we keep nondeterminism in the inputs but remove it from the outputs, 
the problems $\textbf{NDet}(\mathscr{F})\tomq \textbf{Det}(\mathscr{F'})$ do not
become any easier.  When outputs are restricted to be
deterministic, the size of the output can be substantially larger than that of
the input filter, (\MakeUppercase{\romannumeral 4}). If one were to think of
this as a search problem, a more restrictive type can drastically increase 
the search space size.  
It only ever takes polynomial space for $\textbf{Det}(\mathscr{F})\tom
\textbf{Det}(\mathscr{F'})$, 
but it is unclear whether this holds for $\textbf{NDet}(\mathscr{F})\tomq
\textbf{Det}(\mathscr{F'})$, and the increase in output size is unfavourable
(though inconclusive) evidence to the contrary.

\begin{figure}
	\begin{center}
		\includegraphics[width=1\columnwidth]{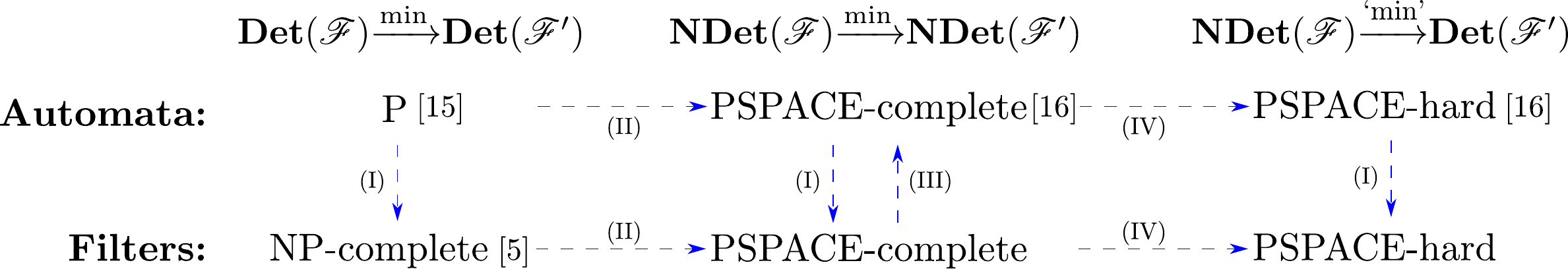}
	\end{center}
	\vspace{-0.2cm}
	\caption{
    \shortenXor{}{\small}
    A comparison between hardness results of decision versions of
automata minimization and filter minimization.%
\label{fig:overview}}
\shortvspace{-12pt}
\end{figure}

\section{Conclusion}
This paper shows the value of nondeterminism in combinatorial filter reduction,
analyzes the hardness of nondeterministic filter minimization problems, and
shows that the hardness separation between deterministic filter minimization
problems disappears in the nondeterministic cases. 

Future work might consider the hardness results for finding nondeterministic
minimizers for deterministic input filters, which is only known to be in \pspace.
Another direction is to examine complete, approximation, and heuristic
algorithms to solve nondeterministic filter minimization.

\bibliographystyle{IEEEtran}
\bibliography{mybibshort}

\end{document}

%% file: defs.tex
\let\emptyset\varnothing

\usepackage{mathrsfs}

\theoremstyle{definition} 
\newtheorem{x}{X} 
 
\newtheorem{z}{Z} 
\newtheorem{definition}[x]{Definition} 
\newtheorem{problem}[x]{Problem} 
\newtheorem{construction}[x]{Construction} 
 
\newtheorem{lemma}[x]{Lemma} 
\newtheorem{theorem}[x]{Theorem} 
\newtheorem{example}[x]{Example}
\newtheorem{property}[x]{Property}

\newtheorem{conjecture}[z]{Conjecture}

\newcommand{\True}{{True}\xspace} 
\newcommand{\False}{{False}\xspace}

\newcommand{\PositiveNaturals}{\mathbb{N}^{+}}

\newcommand{\Language}[2]{\ensuremath{\mathcal{L}^{#1}(#2)}}
\providecommand{\ILanguage}[1]{\Language{I}{#1}}
\providecommand{\ALanguage}[1]{\Language{A}{#1}}
\providecommand{\GLanguage}[1]{\Language{G}{#1}}

\newcommand{\Vgoal}{V_{\mathrm{goal}}}

\providecommand{\Sa}{\ensuremath{\rm{a}}\xspace}
\providecommand{\Sb}{\ensuremath{\rm{b}}\xspace}
\providecommand{\Sc}{\ensuremath{\rm{c}}\xspace}
\providecommand{\Sx}{\ensuremath{\rm{x}}}

\newcommand*{\probleminternal}[4]{
	\par
	\medskip
	\noindent\fbox{\parbox{0.98\columnwidth}{
		\textbf{#4:} {#1} \\[0.05in]
		\renewcommand{\tabcolsep}{2pt}
		\begin{tabularx}{\linewidth}{rX}
			\emph{Input:} & #2 \\
			\emph{Output:} & #3
		\end{tabularx}
	}}
	\par
	\medskip
	\par
}

\let\emptyset\varnothing

\newcommand*{\ourproblem}[3]{\probleminternal{#1}{#2}{#3}{Problem}}
\newcommand*{\decproblem}[3]{\probleminternal{#1}{#2}{#3}{Decision Problem}}
\providecommand{\defemp}[1]{\emph{#1}} 

\newcommand{\pfmdec}{{\sc pfm-dec}\xspace}
\newcommand{\pfm}{{\sc pfm}\xspace}
\newcommand{\pfdm}{{\sc pfdm}\xspace}
\newcommand{\pfdmdec}{{\sc pfdm-dec}\xspace}

\newcommand{\nop}{{\sc nop}\xspace}

\newcommand{\placeholder}{\ensuremath{\ominus}\xspace}

\providecommand{\reachedv}[2]{\ensuremath{\mathcal{V}_{#1}(#2)}}
\providecommand{\reachedc}[2]{\ensuremath{\mathcal{C}(#1, #2)}}

\providecommand{\fprod}{\ensuremath{\odot}}

\providecommand{\oss}{{output simulates}\xspace}

\providecommand{\structure}[1]{\mathscr{#1}\xspace}

\providecommand{\ptext}{{\footnotesize\textsf{P}}}
\providecommand{\nptext}{{\footnotesize\textsf{NP}}}
\providecommand{\pspacetext}{{\footnotesize\textsf{PSPACE}}}

\providecommand{\p}{\ptext\xspace}

\providecommand{\pspace}{\pspacetext\xspace}
\providecommand{\nphard}{{\nptext -hard}\xspace}
\providecommand{\npcomplete}{{\nptext -complete}\xspace}
\providecommand{\pspacehard}{{\pspacetext -hard}\xspace}
\providecommand{\pspacecomplete}{{\pspacetext -complete}\xspace}
\providecommand{\tom}{\ensuremath{\xrightarrow[]{\text{min}}}}
\providecommand{\tomq}{\ensuremath{\xrightarrow[]{\text{`min'}}}}

\providecommand{\fig}{Fig.}
\providecommand{\figs}{Figs.}

%% file: figure/merged_deterministic_input_small.tex
\newcommand\vs{1.5}
\newcommand\vS{3.0}
\newcommand\hsb{1.0}
\newcommand\hsa{2.0}
\newcommand\hsahalf{0.5}
\makeatletter
\tikzset{circle split part fill/.style  args={#1,#2}{%
 alias=tmp@name, 
  postaction={%
    insert path={
     \pgfextra{%
     \pgfpointdiff{\pgfpointanchor{\pgf@node@name}{center}}%
                  {\pgfpointanchor{\pgf@node@name}{east}}%
     \pgfmathsetmacro\insiderad{\pgf@x}
      \fill[#1] (\pgf@node@name.base) ([xshift=-\pgflinewidth]\pgf@node@name.east) arc
                          (0:180:\insiderad-\pgflinewidth)--cycle;
      \fill[#2] (\pgf@node@name.base) ([xshift=\pgflinewidth]\pgf@node@name.west)  arc
                           (180:360:\insiderad-\pgflinewidth)--cycle;            
         }}}}}  

\tikzset{
diagonal fill/.style 2 args={fill=#2, path picture={
\fill[#1, sharp corners] (path picture bounding box.south west) -|
                         (path picture bounding box.north east) -- cycle;}},
reversed diagonal fill/.style 2 args={fill=#2, path picture={
\fill[#1, sharp corners] (path picture bounding box.north west) |- 
                         (path picture bounding box.south east) -- cycle;}}
}
 \makeatother

{
\begin{tikzpicture}[shorten >=1pt,node distance=1cm, on grid, auto, initial text
=] 
\tikzset{every state/.style={semithick, minimum size=10pt, inner
sep=8pt,rectangle}}
\tikzset{every path/.style={thick, font=\Large}}
\tikzset{initial distance=1.2cm}

\node[state, initial, fill=white] (q0)   {}; 
   \node[state] (q1) [below left=\vs and 6cm of q0, fill=cyan!20] {}; 
   \node[state] (q2) [right=3cm of q1,              fill=cyan!20] {}; 
   \node[state] (q5) [right=3cm of q2,              fill=cyan!20] {}; 
   \node[state] (q3) [right=3cm of q5,              fill=cyan!20] {}; 
   \node[state] (q4) [right=3cm of q3,              fill=cyan!20] {}; 
   \node[state] (q6) [below left=\vs cm and \hsb cm of q1, fill=green!50] {}; 
   \node[state] (q7) [below =\vs cm of q1, diagonal fill={teal!50}{violet!50}] {}; 
   \node[state] (q8) [below right=\vs cm and 1.5cm of q1, fill=red!50] {}; 
   \node[state] (q9) [below =\vs cm of q2, fill=brown!90] {}; 
   \node[state] (q10) [below right=\vs cm and \hsb cm of q2, diagonal
fill={violet!50}{lime!20}] {}; 

   \node[state] (q12) [below =\vs cm  of q3, fill=lime!20] {}; 
   \node[state] (q13) [below left=\vs cm and 1.5 cm of q4, fill=yellow!75] {}; 
   \node[state] (q15) [below right =\vs cm and \hsb cm of q4, fill=teal!50] {}; 
   \node[state] (q16) [below left=\vs cm and \hsb cm of q5, fill=olive!40] {}; 
   \node[state] (q17) [below =\vs cm of q5, diagonal fill={yellow!75}{red!50}] {}; 
   \node[state] (q18) [below right=\vs cm and \hsb cm of q5, fill=violet!50] {}; 
   \node[state] (q11) [below left =\vs cm and \hsb of q3, diagonal fill={brown!90}{olive!40}] {}; 
   \node[state] (q14) [below=\vs cm of q4, diagonal fill={olive!40}{green!50}] {}; 

    \path[->] 
    (q0) edge node [pos=0.7, sloped, above] {\rm a} (q1)
    (q0) edge node [pos=0.7, sloped, above] {\rm b} (q2)
    (q0) edge node [pos=0.7, sloped, above] {\rm d} (q3)
    (q0) edge node [pos=0.7, sloped, above] {\rm e} (q4)
    (q0) edge node [pos=0.4, sloped, right, rotate=90] {\rm c} (q5)

    (q1) edge node [pos=0.5, sloped, above] {\rm a} (q6)
    (q1) edge node [pos=0.5, sloped, right, rotate=90] {\rm c} (q7)
    (q1) edge node [pos=0.5, sloped, above] {\rm b} (q8)
    (q2) edge node [pos=0.5, sloped, above] {\rm b} (q8)
    (q2) edge node [pos=0.5, sloped, left, rotate=90] {\rm a} (q9)
    (q2) edge node [pos=0.5, sloped, above] {\rm c} (q10)
    (q3) edge node [pos=0.5, sloped, right, rotate=90] {\rm c} (q12)
    (q3) edge node [pos=0.5, sloped, above] {\rm b} (q13)
    (q4) edge node [pos=0.5, sloped, above] {\rm b} (q13)
    (q4) edge node [pos=0.5, sloped, above] {\rm c} (q15)
    (q5) edge node [pos=0.5, sloped, above] {\rm a} (q16)
    (q5) edge node [pos=0.5, sloped, right, rotate=90] {\rm b} (q17)
    (q5) edge node [pos=0.5, sloped, above] {\rm c} (q18);

    \path[->] 
    (q3) edge node [pos=0.5, sloped, above] {\rm a} (q11)
    (q4) edge node [pos=0.5, sloped, right, rotate=90] {\rm a} (q14);

    \end{tikzpicture}
}

%% file: figure/smo_minimizer_small.tex
\newcommand\vs{1.5}
\newcommand\vS{2.1}
\newcommand\hs{1.0}
\newcommand\hsb{1.0}
\newcommand\hsa{2.5}
\newcommand\hsbf{3.5}

{
\begin{tikzpicture}[shorten >=1pt,node distance=1cm, on grid,auto]
\tikzset{every state/.style={semithick, minimum size=10pt, inner sep=8pt,
rectangle}}
\tikzset{every path/.style={thick, font=\Large}}
\tikzset{initial distance=1.2cm, initial text=}
\node[state, initial, fill=white] (q0)   {}; 
\node[state] (p1) [below left=\vs cm and 5cm of q0, fill=cyan!20] {}; 
   \node[state] (p2) [right=\hsbf cm of p1, fill=cyan!20] {}; 
   \node[state] (p3) [right=\hsbf cm of p2, fill=cyan!20] {}; 
   \node[state] (p4) [right=\hsbf cm of p3, fill=cyan!20] {}; 
   \node[state] (p5) [below left=\vs and \hsb of p1, fill=green!50] {}; 
   \node[state] (p6) [below right=\vs and \hsb of p1, fill=teal!50] {}; 
   \node[state] (p7) [below left=\vs and \hsb of p2,  fill=red!50] {}; 
   \node[state] (p8) [below right=\vs and \hsb of p2, fill=violet!50] {}; 
   \node[state] (p9) [below left=\vs and \hsb of p3,  fill=brown!90] {}; 
   \node[state] (p10) [below right=\vs and \hsb of p3,fill=lime!20] {}; 
   \node[state] (p11) [below left=\vs and \hsb of p4, fill=yellow!75] {}; 
   \node[state] (p12) [below right=\vs and \hsb of p4,fill=olive!40] {}; 
   
    \path[->] 
    (q0) edge node [pos=0.7, sloped, above, inner sep=0.2ex] {\rm a,e} (p1)
    (q0) edge node [pos=0.5, sloped, below, inner sep=0.2ex] {\rm a,b,c} (p2)
    (q0) edge node [pos=0.7, sloped, above, inner sep=0.2ex] {\rm b,d} (p3)
    (q0) edge node [pos=0.7, sloped, above, inner sep=0.2ex] {\rm c,d,e} (p4)
    (p1) edge node [pos=0.7, above, inner sep=1.8ex] {\rm a} (p5)
    (p1) edge node [pos=0.7, above, inner sep=1.8ex] {\rm c} (p6)
    (p2) edge node [pos=0.7, above, inner sep=1.8ex] {\rm b} (p7)
    (p2) edge node [pos=0.7, above, inner sep=1.8ex] {\rm c} (p8)
    (p3) edge node [pos=0.7, above, inner sep=1.8ex] {\rm a} (p9)
    (p3) edge node [pos=0.7, above, inner sep=1.8ex] {\rm c} (p10)
    (p4) edge node [pos=0.7, above, inner sep=1.8ex] {\rm b} (p11)
    (p4) edge node [pos=0.7, above, inner sep=1.8ex] {\rm a} (p12);

\end{tikzpicture}
}